\documentclass[journal]{IEEEtran}
\usepackage{cite}
\ifCLASSINFOpdf
\else
\fi
\usepackage{amsmath,amssymb}
\usepackage{algorithmic}
\usepackage{array}
\usepackage{url}
\usepackage{amsthm}
\usepackage{booktabs}
\usepackage{graphicx}
\usepackage{epstopdf}
\usepackage{subfigure}

\usepackage{multirow}
\usepackage{gensymb}
\newtheorem{prop}{Proposition}

\newtheorem{proper}{Property}

\begin{document}

\def\o{\ensuremath\varnothing}
\title{Principal Component Analysis Based on \\T$\ell_1$-norm Maximization}

\author{Xiang-Fei~Yang,~\IEEEmembership{}
        Yuan-Hai~Shao,~\IEEEmembership{}
        Chun-Na~Li,~\IEEEmembership{}
        Li-Ming~Liu,~\IEEEmembership{}
        and~Nai-Yang~Deng ~\IEEEmembership{}% <-this % stops a space

\thanks{This work was supported in part by the National Natural Science Foundation of China under Grant 11926349, Grant 61866010, Grant 11871183, and Grant 61703370, in part by the Scientic Research Foundation of Hainan University under Grant KYQD(SK)1804, and in part by Heavy Beijing City Research Center Project under Grant TDJD201502. \textit{(Corresponding authors: Yuan-hai Shao and Li-Ming Liu.)}}% <-this % stops a space
\thanks{X.-F.~Yang is a Ph.D student at School of Statistics, Capital University of Economics and Business, Beijing 100070, China (e-mail: yxf9011@163.com).}% <-this % stops a space
\thanks{Y.-H.~Shao is with the School of Management, Hainan University, Haikou 570228, China (e-mail: shaoyuanhai21@163.com).}% <-this % stops a space
\thanks{C.-N.~Li is with the School of Management, Hainan University, Haikou 570228, China (e-mail: na1013na@163.com).}% <-this % stops a space
\thanks{L.-M.~Liu is with the School of Statistics, Capital University of Economics and Business, Beijing 100070, China (e-mail: llm5609@163.com).}% <-this % stops a space
\thanks{N.-Y.~Deng was with the College of Science, China Agricultural University, Beijing 100083, China (e-mail: dengnaiyang@cau.edu.cn).}% <-this % stops a space
\thanks{This paper has supplementary material in appendix.}
\thanks{Manuscript received April 19, 2005; revised August 26, 2015.}}


\markboth{Journal of \LaTeX\ Class Files,~Vol.~14, No.~8, August~2015}%
{Shell \MakeLowercase{\textit{et al.}}: Bare Demo of IEEEtran.cls for IEEE Journals}



\maketitle

\begin{abstract}
Classical principal component analysis (PCA) may suffer from the sensitivity to outliers and noise. Therefore PCA based on $\ell_1$-norm and $\ell_p$-norm ($0 < p < 1$) have been studied. Among them, the ones based on $\ell_p$-norm seem to be most interesting from the robustness point of view. However, their numerical performance is not satisfactory. Note that, although T$\ell_1$-norm is similar to $\ell_p$-norm ($0 < p < 1$) in some sense, it has the stronger suppression effect to outliers and better continuity. So PCA based on T$\ell_1$-norm is proposed in this paper. Our numerical experiments have shown that its performance is superior than PCA-$\ell_p$ and $\ell_p$SPCA as well as PCA, PCA-$\ell_1$ obviously.
\end{abstract}

\begin{IEEEkeywords}
Principal component analysis (PCA), T$\ell_1$-norm, Robust modeling, Dimensionality reduction.
\end{IEEEkeywords}

% \fi
\IEEEpeerreviewmaketitle

\section{Introduction}
\IEEEPARstart{P}{rincipal} component analysis (PCA)\cite{Hotelling1933,Jollife2002}, a popular toolkit for data processing and pattern recognition, has been widely investivated during the last decades. It is often utilized for dimensionality reduction. PCA tries to find a set of projection vectors consisting of the linear combinations of the given data that either maximizes the dispersion of the projected data or minimizes the reconstruction error. These projection vectors construct a low-dimensional subspace that can capture the intrinsic structure of the original data.

However, classical PCA has a fatal drawback. It is sensitive to outliers because using $\ell_2$-norm metric. To overcome this problem, $\ell_2$-norm is substituted by $\ell_1$-norm. Baccini $et$ $al$. \cite{Baccini1996} proposed a PCA based on $\ell_1$-norm ($\ell_1$-PCA) by minimizing the reconstruction error, and correspondingly, a heuristic algorithm based on maximum likelihood estimation was presented. Subsequently, the weighted median alogrithm and the quadratic programming algorithm were proposed in \cite{Ke2003}, where the robustness with $\ell_1$-norm was also addressed. Noticing that the above $\ell_1$-PCA methods are not rotational invariant. Ding $et$ $al$. \cite{Ding2006} proposed a rotational invariant $\ell_1$-norm PCA ($R_1$-PCA) which combines the merits of $\ell_2$- and $\ell_1$-norm PCA. However, Kwak\cite{Kwak2008} pointed out that $R_1$-PCA depends highly on the dimension of  a subspace to be founded. And in \cite{Kwak2008}, Kwak also proposed a PCA based on $\ell_1$-norm (PCA-$\ell_1$) by maximizing the dispersion. The PCA-$\ell_1$ is a greedy method with easy implementation. Then Nie $et$ $al$. proposed its non-greedy version with better experimental results in \cite{Nie2011}. Unlike the aforementioned methods where only the local solution can be obtained, another PCA based on $\ell_1$-norm proposed in \cite{Brooks2013} could find a global solution. In addition, The other PCA methods based on $\ell_1$-norm are concerned with the sparseness, regularization, kernel trick and two-dimensional problem (2D)\cite{Meng2012,Lai2014,Wang2013,Lu2016,Kim2019,Fan2020,Liwicki2013}.

To further improve the robustness, some reaschers noticed the $\ell_p$-norm. Liang $et$ $al$.\cite{Liang2103} proposed the generalized PCA based on $\ell_p$-norm  ($\ell_p$-norm GPCA), where the $\ell_p$-norm  was employed to be as constraint instead of the objective function. In \cite{Kwak2014}, Kwak extended PCA-$\ell_1$ to PCA-$\ell_p$ for an arbitrary $p>0$ and proposed both the greedy and non-greedy algorithms. The proposed algorithms are convergent under the condition of $p \geq 1$. The other PCA methods based on $\ell_p$-norm are concerned with low-rank technique, sparseness and 2D problem\cite{Wang2016,Quach2017,Li2017}. It is naturally believed that $\ell_p$-norm is more robust to $\ell_1$-norm when $0<p<1$, but it does not satisfy Lipschitz-continuity which is important for robustness \cite{Weng2018,Cranko2018}. And most of $\ell_p$-norm PCA methods have been shown to be non-monotonic when $0 < p < 1$. These all restricted the applications of the $\ell_p$-norm PCAs.

In this paper, to give a more robust PCA with Lipschitz-continuity measurement, the T$\ell_1$-norm is studied. Indeed, T$\ell_1$-norm is similar to $\ell_p$-norm ($0<p<1$) in some sense, but it has the stronger suppression effect to outliers and better continuity. Using this norm, we proposed a PCA based on T$\ell_1$-norm (T$\ell_1$PCA) by maximizing T$\ell_1$-norm-based dispersion in the projection space. Correspondingly, to solve the optimization problem, a modified ascent method on sphere is constructed. The results of the preliminary experiments show that T$\ell_1$PCA is more robust than some current PCAs based on $\ell_1$-norm and $\ell_p$-norm.

The rest of this paper is organized as follows. In Section \uppercase\expandafter{\romannumeral2}, we introduce and analyze the T$\ell_1$-norm. In Section \uppercase\expandafter{\romannumeral3}, the optimization problem of our PCA based on T$\ell_1$-norm is formulated. To solve the optimization problem, an ascend method is constructed and investigated in Section \uppercase\expandafter{\romannumeral4}. In Section \uppercase\expandafter{\romannumeral5}, T$\ell_1$PCA is applied to several artifical and real datasets and the performances are compared with some other current PCA methods. Finally, the conclusion follows in Section \uppercase\expandafter{\romannumeral6}.

\section{T$\ell_1$-norm}
In this section, based on the transformed $\ell_1$ (T$\ell_1$) penality function \cite{Nikolova2000,Lv2009,Zhang2018,Zhang2017_1,Zhang2017_2,Ma2019}, T$\ell_1$-norm is  introduced: for a vector $\boldsymbol{x} =[x_1,...,x_n]^T\in\mathbb{R}^{n}$, we define T$\ell_1$-norm as
\begin{equation}
\begin{split}
||\boldsymbol{x}||_{T\ell_1(a)}=||\boldsymbol{x}||_{T\ell_1}=\sum\limits_{i=1}^{n}{\rho_a(x_i)},\\
\end{split}
\end{equation}
where $\rho_a(\cdot)$ is the operator of component
\begin{equation}\label{rho_a}
\begin{split}
\rho_a(t)=\frac{(a+1)|t|}{a+|t|}, \\
\end{split}
\end{equation}
and $a$ is a positive shape parameter. Generally speaking, the norm should satisfy the following three properties: \romannumeral1) Positive definite: for all $\boldsymbol{x} \in \mathbb{R}^n$, $||\boldsymbol{x}|| \geq 0$ and $||\boldsymbol{x}||=0$ iff $\boldsymbol{x}=0$; \romannumeral2) Triangle inequality: for all $\boldsymbol{x}, \boldsymbol{y} \in \mathbb{R}^n$, $||\boldsymbol{x}+\boldsymbol{y}|| \leq ||\boldsymbol{x}||+||\boldsymbol{y}||$; \romannumeral3) Absolutely homogeneity: for all $\boldsymbol{x} \in \mathbb{R}^n$ and scalar $c$, $||c\boldsymbol{x}||=|c|\cdot||\boldsymbol{x}||$. And $||\cdot||$ means the general form of norms. Obviously, T$\ell_1$-norm satisfies the first two properties but not satisfies the third one. So, strictly speaking, T$\ell_1$-norm is not a norm. But in this paper, we still call it T$\ell_1$-norm for convenience.

Further, we discuss the properties of T$\ell_1$-norm and compare them with those of $\ell_p$-norm ($0 \leq p \leq 1$). The $\ell_p$-norm of a vector $\boldsymbol{x}=[x_1,...,x_n]^T\in\mathbb{R}^{n}$ is denoted as
\begin{equation}
\begin{split}
\ ||\boldsymbol{x}||_p=\left( \sum\limits_{i=1}^{n}{\mu_p(x_i)} \right)^{1/p},\\
\end{split}
\end{equation}
where $\mu_p(\cdot)$ is the operator of component
\begin{equation}\label{mu_p}
\begin{split}
\mu_p(t)=|t|^p. \\
\end{split}
\end{equation}
It is known that, e. g. see \cite{Ma2019}, T$\ell_1(a)$-norm is related with $\ell_p$-norm in the following way: for any vector $\boldsymbol{x}=[x_1,...,x_n]^T\in\mathbb{R}^{n}$, with the change of parameter $a$, T$\ell_1(a)$-norm interpolates $\ell_0$-norm and $\ell_1$-norm as
\begin{equation}\label{L12}
\begin{split}
\underset{a\rightarrow 0^+}{\rm{lim}} ||\boldsymbol{x}||_{T\ell_1}=||\boldsymbol{x}||_0, \\
\underset{a\rightarrow\infty}{\rm{lim}} ||\boldsymbol{x}||_{T\ell_1}=||\boldsymbol{x}||_1.\\
\end{split}
\end{equation}
To show the similarity between T$\ell_1$-norm and $\ell_p$-norm, their contours with $a = 10^{-2}, 1, 10^2$ and $p = 0, \frac{1}{2}, 1$ are plotted in \cite{Ma2019}. From the set of figures, it is concluded that T$\ell_1$-norm with $a = 10^{-2}, 1$ and $10^2$ indeed approximates $\ell_0$-norm, $\ell_{1/2}$-norm and $\ell_1$-norm, respectively. This seems to imply that a one-to-one relationship exists between $a$ and $p$ making T$\ell_1(a)$-norm approximate $\ell_p$-norm. For example, $a=1$ corresponds to $p=\frac{1}{2}$ and T$\ell_1(1)$-norm approximates $\ell_{1/2}$-norm. However, investigating the definitions of T$\ell_1$-norm and $\ell_p$-norm carefully, we do find their severe difference. In fact, we need only to compare their component operators as shown in the following property.

\begin{proper}\label{property_1}For any fixed $a$ $(a>0)$ and $p$ $(0<p<1)$, comparing $\rho_a(t)$ with $\mu_{p}(t)$, there exist the following conclusions: the function $\rho_{a}(t)$ is Lipschitz-continuous with the Lipschitz constant $1+a^{-1}$. When $|t|$ increases from 0 to infinity, the function value $\rho_{a}(t)$ increases from 0 to a finite value $a+1$. However, $\mu_{p}(t)$ is not Lipschitz-continuous. When $|t|$ increases from 0 to infinity, the function value $\mu_p(t)$ increases from 0 to infinity.
\end{proper}

\vspace{-0.5cm}
\begin{figure}[!htp]
\centering
\subfigure[]{\includegraphics[width=0.175\textheight]{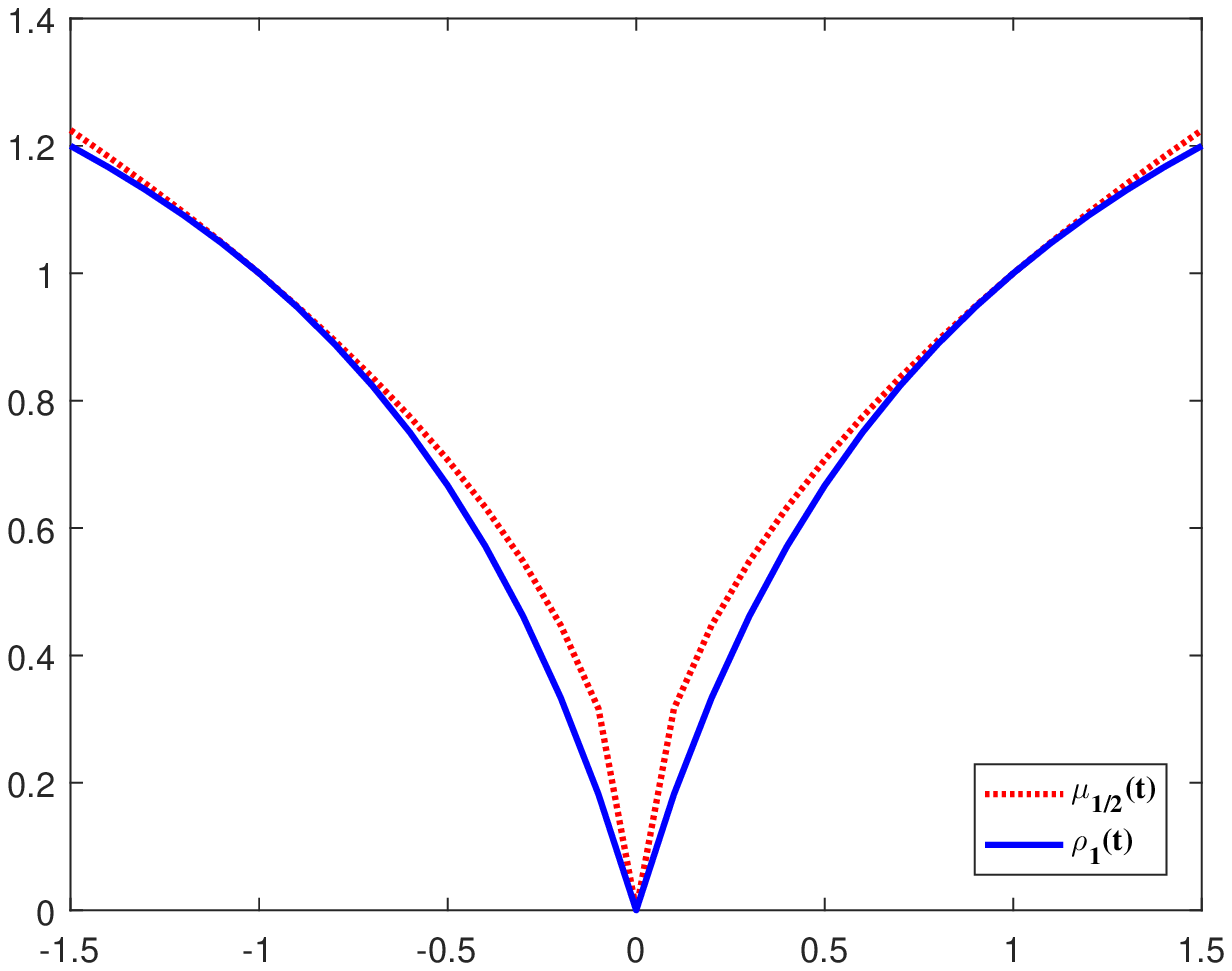}}
\subfigure[]{\includegraphics[width=0.175\textheight]{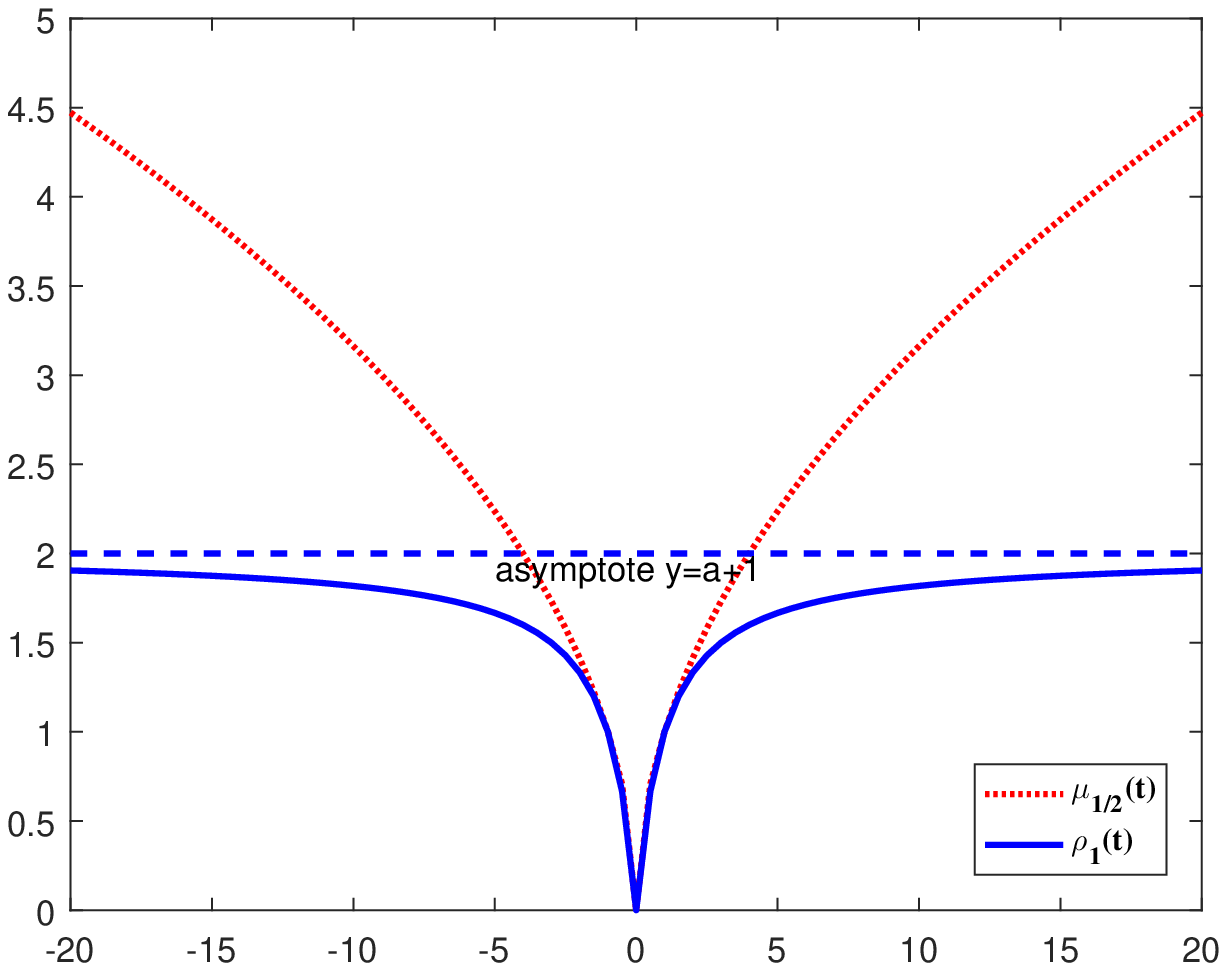}}
\setlength{\abovecaptionskip}{0cm}
\caption{The differences of $\rho_1(t)$ and $\mu_{1/2}(t)$. (a) For small value of $|t|$. (b) For large value of $|t|$.} \label{pTL1_fig}
\end{figure}
\vspace{-0.35cm}

Note that the above property points out the difference between T$\ell_1$-norm with any $a \in \left( 0, \infty \right)$ and $\ell_p$-norm with any $p \in \left( 0, 1 \right)$, including $a = 1$ and any $p \in \left( 0, 1 \right)$, particularly $a = 1$ and $p = \frac{1}{2}$. For the last case, both the component operators $\rho_{a}(t) = \rho_{1}(t)$ and $\mu_{p}(t) = \mu_{1/2}(t)$ are shown in Fig. \ref{pTL1_fig}, where Fig. \ref{pTL1_fig} (a) and Fig. \ref{pTL1_fig} (b) indicate their difference when $|t|$ is small or large, respectively. Corresponding to Fig. \ref{pTL1_fig} (a), we have  $\underset{t\rightarrow 0^+}{\rm{lim}} \rho_{1}^{'}(t)=2$ and $\underset{t\rightarrow 0^+}{\rm{lim}} \mu_{1/2}^{'}(t)=\infty$. And corresponding to Fig. \ref{pTL1_fig} (b), we have $\underset{t\rightarrow \infty}{\rm{lim}} \rho_{1}(t)=2$ and $\underset{t\rightarrow \infty}{\rm{lim}} \mu_{1/2}(t)=\infty$. So, there is a marked difference between T$\ell_1$-norm and $\ell_p$-norm whether $|t|$ is small or lagre.
\vspace{-0.5cm}
\begin{figure}[!htp]
\centering
\subfigure{\includegraphics[width=0.28\textheight]{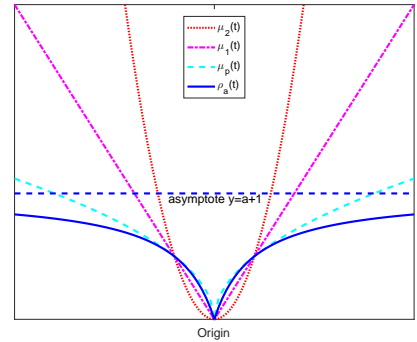}}
\setlength{\abovecaptionskip}{-0.35cm}
\caption{The comparisons of $\mu_{2}(t)$, $\mu_{1}(t)$, $\mu_{p}(t)$ and $\rho_a(t)$. The functions $\mu_{2}(t)$, $\mu_{1}(t)$ and $\mu_{p}(t)$ increase from 0 to $\infty$. The function $\rho_a(t)$ has an asymptote $y=a+1$.}  \label{allTL1_fig}
\end{figure}
\vspace{-0.2cm}

The above discussion implies the advantages of applying T$\ell_1$-norm in robust problem. In fact, retrospect the development course of the norm in PCA: from $\ell_2$-norm to $\ell_1$-norm and then to $\ell_p$-norm; and their corresponding component operators  from $\mu_{2}(t)=|t|^2$ to $\mu_{1}(t)=|t|$ and then to $\mu_{p}(t)=|t|^p$. Fig. \ref{allTL1_fig} shows the figures of these three operators. Obviously, for large $|t|$, when $|t|$ increases, the growth slows down gradually from $\mu_{2}(t)=|t|^2$ to $\mu_{1}(t)=|t|$ and then to $\mu_{p}(t)=|t|^p$. Fig. \ref{allTL1_fig} also shows the figure of the component operator $\rho_{a}(t)$ in T$\ell_1$-norm and the function $\rho_{a}(t)$ with fixed $a$ is bounded. And from $\mu_{p}(t)$ to $\rho_{a}(t)$, the growth slows further. This means that T$\ell_1$-norm  has better suppression effect to outliers. In addition, as discussed above, $\rho_{a}(t)$ has better continuity than $\mu_{p}(t)$, especially its Lipschitz continuity which is good for robustness. Therefore, it can be expected to have better robustness by using T$\ell_1$-norm in PCA than $\ell_p$-norm.

\section{Problem formulation}
Let $\mathbf{X}=[\boldsymbol{x}_1,\cdots,\boldsymbol{x}_n]\in\mathbb{R}^{d\times n}$ be a given data matrix, where $d$ and $n$ denote the dimension of the original space and the number of sapmles respectively. Without loss of generality, suppose the data $\{\boldsymbol{x}_i\}_{i=1}^{n}$ has been centralized, i.e. ,$\sum\limits_{i=1}^{n}{\boldsymbol{x}_i}=0$.

Firstly, we consider the following general PCA maximization problem
\begin{equation}\label{LpPCA model}
\begin{split}
&\underset{\mathbf{W}}{\max}~\sum\limits_{i=1}^{n}{||\boldsymbol{x}_i^T\mathbf{W}||}\\
&~~\hbox{s.t.\ }\mathbf{W}^{T}\mathbf{W}=\mathbf{I},
\end{split}
\end{equation}
where $\mathbf{W}= [\boldsymbol{w}_1,\cdots,\boldsymbol{w}_m]\in\mathbb{R}^{d\times m}$ is the projection matrix consisted of $m$ projection vectors. When $||\cdot||$ is subtituted by $\ell_2$-, $\ell_1$-, and $\ell_p$-norm, it is identical to $\ell_2$-PCA, PCA-$\ell_1$ and PCA-$\ell_p$ respectively. However, it is difficult to solve equation \eqref{LpPCA model} directly for some norms. To address this problem, it is simplified into a series of $m=1$ problems and \eqref{LpPCA model} becomes the following optimization probelm
\begin{equation}\label{LpPCA model2}
\begin{split}
\underset{\boldsymbol{w}}{\max}&~~{||\mathbf{X}^{T}\boldsymbol{w}||}\\
\hbox{s.t.\ }&\boldsymbol{w}^{T}\boldsymbol{w}=1,
\end{split}
\end{equation}
When $m>1$, greedy method could be utilized to solve.

In this paper, we employ T$\ell_1$-norm into equation \eqref{LpPCA model} and construct the PCA based on T$\ell_1$-norm as follows
\begin{equation}\label{TL1PCA model}
\begin{split}
\underset{\mathbf{W}}{\max}&~\sum\limits_{i=1}^{n}{||\boldsymbol{x}_i^{T}\mathbf{W}||_{T\ell_1}}\\
\hbox{s.t.\ }&\mathbf{W}^{T}\mathbf{W}=\mathbf{I}.
\end{split}
\end{equation}
When $m>1$, it is also difficult to find an optimal solution of \eqref{TL1PCA model}. We also simplify the problem into a series of $m=1$ optimization probelms by using a greedy method, therefore, we will first solve the following optimization problem
\begin{equation}\label{TL1PCA model2}
\begin{split}
\underset{\boldsymbol{w}}{\max}&~~{f(\boldsymbol{w})=||\mathbf{X}^{T}\boldsymbol{w}||_{T\ell_1}}\\
&~~\hbox{s.t.\ }\boldsymbol{w}^T\boldsymbol{w}=1 ,
\end{split}
\end{equation}
which is equivalent to
\begin{equation}\label{TL1PCA model3}
\begin{split}
\underset{\boldsymbol{w}}{\max}&~~{f(\boldsymbol{w})=\sum\limits_{i=1}^{n}{\frac{(a+1)|\boldsymbol{x}_i^T\boldsymbol{w}|}{a+|\boldsymbol{x}_i^T\boldsymbol{w}|}}}\\
&~~~~~~\hbox{s.t.\ }\boldsymbol{w}^T\boldsymbol{w}=1.
\end{split}
\end{equation}

\section{Algorithm}
Since the problem \eqref{TL1PCA model} is non-convex and non-smooth, the traditional convex optimization technique could not be used directly. Therefore, we first consider to solve problem \eqref{TL1PCA model3}, which is a relatively simple situation of problem \eqref{TL1PCA model}. Even so, it is also difficult to solve \eqref{TL1PCA model3} since it has the division operator of absolute value functions. Although the alternating direction method of multipliers (ADMM) \cite{Boyd2011, Li2020} is a popular method to solve non-convex and non-smooth problem, it does not apply to our problem because of the constraint $\boldsymbol{w}^T\boldsymbol{w} = 1$. Motivated by the methods in \cite{Yu2016} and \cite{Boyd2004}, we design a modified gradient ascent method on a sphere to solve \eqref{TL1PCA model3}. And the method could guarantee the constraint.

Here, we need to compute the gradient of $f(\boldsymbol{w})$ with respect to $\boldsymbol{w}$ as follows
\begin{equation}\label{Gradient}
\begin{split}
\nabla f(\boldsymbol{w})=\sum\limits_{i=1}^{n}{\frac{a(a+1)sign(\boldsymbol{x}_i^T\boldsymbol{w})\boldsymbol{x}_i}{(a+|\boldsymbol{x}_i^T\boldsymbol{w}|)^2}},\\
\end{split}
\end{equation}
where
$$sign(t)= \begin{cases} 1,& t > 0\\ 0,& t=0\\-1,& t<0 \end{cases}, $$
and a random positive vector is added on $\boldsymbol{w}$ to satisfy $\boldsymbol{x}_i^T\boldsymbol{w} \neq 0$ when $\boldsymbol{x}_i^T\boldsymbol{w}=0$. Then we project $\nabla f(\boldsymbol{w})$ onto the tangent plane of $\boldsymbol{w}$ on the unit sphere as $\boldsymbol{g}= \nabla f(\boldsymbol{w})-\langle\nabla f(\boldsymbol{w}),\boldsymbol{w}\rangle \boldsymbol{w}$ and normalize it as $\boldsymbol{g}_0=\boldsymbol{g}/||\boldsymbol{g}||_{2}$, where $\langle\cdot, \cdot\rangle$ denotes as inner product of vectors and the unit sphere is determined by the constrain $\boldsymbol{w}^T\boldsymbol{w} = 1$. For the $t$-th step, $\boldsymbol{w}(t)^T\boldsymbol{w}(t) = 1$ and $\boldsymbol{w}(t)^T\boldsymbol{g}_0(t) = 0$, then we have the following update rule
$$\boldsymbol{w}(t+1)=\boldsymbol{w}(t)\cos(\theta_t)+\boldsymbol{g}_0(t)\sin(\theta_t),$$
where $\theta_t$ controls the step size.

\begin{center}
\begin{tabular}{l}
\toprule
\noindent{$\mathbf{Algorithm~1. Algorithm~for~Solving~(\ref{TL1PCA model3})}$}\\
\midrule
\textbf{Input}: The data matrix $\mathbf{X} \in\mathbb{R}^{d\times n}$, the parameter $a$ of\\
\qquad~~        T$\ell_1$-norm. \\
\textbf{Output}: The projection vector $\boldsymbol{w}$.\\
\textbf{Initialization}: Find $k^{*}=\underset{1\le k\le n}{argmax}~f(\boldsymbol{x}_k/||\boldsymbol{x}_k||_{2})$, where  \\
\qquad~~~~~~~~~~~ $f(\boldsymbol{w})=\sum\limits_{i=1}^{n}{\frac{(a+1) |\boldsymbol{x}_{i}^{T} \boldsymbol{w}|}{a+|{\boldsymbol{x}_i}^T\boldsymbol{w}|}}$.\\
\qquad~~~~~~~~~~~ Set $\boldsymbol{w}(0)=\boldsymbol{x}_{k^*}/||\boldsymbol{x}_{k^*}||_{2}$. Give $\theta_0\in(0,\pi/2]$\\
\qquad~~~~~~~~~~~   randomly.\\
\textbf{Repeat}:\\
%\quad If $\hbox{x}_i^T\hbox{w}(t)=0$ for some $i$ \\
%\qquad $\hbox{w}(t)\leftarrow(\hbox{w}(t)+\delta)/||\hbox{w}(t)+\delta||$, where $\delta$ is small \\
%\qquad random vector. \\
%\quad End if;\\
\quad Compute the gradient $\nabla f(\boldsymbol{w}(t))$ of $f$ at $\boldsymbol{w}(t)$ by \eqref{Gradient};  \\
\quad If $\boldsymbol{w}(t)$ and $\nabla f(\boldsymbol{w}(t))$ are collinear \\
\qquad $\nabla f(\boldsymbol{w}(t)) \leftarrow \nabla f(\boldsymbol{w}(t))+ \boldsymbol{\xi} $, where $\boldsymbol{\xi}$ is the perturb-\\
\qquad ation satisfying that $\nabla f(\boldsymbol{w}(t))^T \boldsymbol{\xi}>0$.\\
\quad End if;\\
\quad Project $\nabla f(\boldsymbol{w}(t))$ onto the tangent plane of $\boldsymbol{w}(t)$,  \\
\quad i.e., $\boldsymbol{g}(t)= \nabla f(\boldsymbol{w}(t))-\langle \nabla f(\boldsymbol{w}(t)),\boldsymbol{w}(t)\rangle \boldsymbol{w}(t)$, then \\
\quad normalize $\boldsymbol{g}(t)$, $\boldsymbol{g}_0(t) = \boldsymbol{g}(t)/||\boldsymbol{g}(t)||_{2}$;\\
\quad Update $\boldsymbol{w}(t+1)=\boldsymbol{w}(t)\cos(\theta_t)+\boldsymbol{g}_0(t)\sin(\theta_t)$. \\
\quad Repeat: \\
\quad~~ $\theta_t\leftarrow\theta_t/2$  \\
\quad Until $f(\boldsymbol{w}(t+1))\ge f(\boldsymbol{w}(t))$; \\
\quad Update $\theta_{t+1} = min(2\theta_t,\pi/2)$;\\
\textbf{Until} convergence   \\
\bottomrule
\end{tabular} \label{Algorithm1}\\
\end{center}

The above update rule guarantees that $\boldsymbol{w}(t+1)$ remains of unit length. However, when $\boldsymbol{w}$ and $\nabla f(\boldsymbol{w})$ are collinear, it is not applicable. Inspired by noisy gradient descent algorithm (NGD) \cite{Jain2017}, we add a perturbation  to $\nabla f(\boldsymbol{w})$ to escape this problem. In addition, to accelerate the convergence, $\theta_{t}$ is chosen as an adaptive step size \cite{Yu2016}. The details are described in Algorithm 1. And for Algorithm 1, we have the following proposition.

\begin{prop}\label{Convergence Alg1} The Algorithm 1 will monotonically increase the objective of the problem \eqref{TL1PCA model3} in each iteration.
\end{prop}

\begin{proof}
As we know, $\nabla f(\boldsymbol{w}(t))$ is the fastest ascent direction. When $\boldsymbol{w}(t)$ and $\nabla f(\boldsymbol{w}(t))$ are collinear, we set
$$\nabla f(\boldsymbol{w}(t)) \leftarrow \nabla f(\boldsymbol{w}(t))+ \boldsymbol{\xi} ,$$
it is clear that $\nabla f(\boldsymbol{w}(t))$  is still an ascent direction after adding the perturbation $\boldsymbol{\xi}$, because $\boldsymbol{\xi}$ satisfies $\nabla f(\boldsymbol{w}(t))^T \boldsymbol{\xi}>0$.

Then by projecting $\nabla f(\boldsymbol{w}(t))$ onto the tangent plane of $\boldsymbol{w}(t)$, we obtain $ \boldsymbol{g}(t)= \nabla f(\boldsymbol{w}(t))-\langle \nabla f(\boldsymbol{w}(t)),\boldsymbol{w}(t)\rangle \boldsymbol{w}(t) $ and normalize it as $\boldsymbol{g}_0(t) = \boldsymbol{g}(t)/||\boldsymbol{g}(t)||_{2}$. Since $\langle \boldsymbol{g}(t), \nabla f(\boldsymbol{w}(t))\rangle=||\nabla f(\boldsymbol{w}(t))||_2^2(1-cos^2(\alpha)) \geq 0$, where $\alpha$ is the angle between $\nabla f(\boldsymbol{w}(t))$ and $\boldsymbol{w}(t)$, the direction $\boldsymbol{g}_0(t)$ is also an ascent direction. Then we have the update rule
$$\boldsymbol{w}(t+1)=\boldsymbol{w}(t)\cos(\theta_t)+\boldsymbol{g}_0(t)\sin(\theta_t),$$
where $\theta_t \in (0,\pi/2]$ is the step size. And we set $\theta_{t}\leftarrow \theta_{t}/2$ until $f(\boldsymbol{w}(t+1)) \geq f(\boldsymbol{w}(t))$. Since $\boldsymbol{w}(t)$ and $\boldsymbol{g}_0(t)$ are orthogonal, the update rule keeps the unit vector of $\boldsymbol{w}(t+1)$. To acclerate the convergence, we set $\theta_{t+1} = min(2\theta_t,\pi/2)$ for the next iteration.
\end{proof}

As the objective of problem \eqref{TL1PCA model3} has an upper bound, proposition 1 indicates that the Algorithm 1 is convergent.

Now we can obtain the first projection vector $\boldsymbol{w}_1$ by calling Algorithm 1. To solve more than one projection vectors, we use a genernal orthogonalization method to compute the remaining vectors. Firstly, we give the details of our orthogonalization procedure in Algorithm 2. Then, using the inductive method, proposition 2 shows that the projection vectors solved by Algorithm 2 are strictly orthogonal. Its proof also describes the details of Algorithm 2.
%  算法2三线表
\begin{center}
\begin{tabular}{l}
\toprule
\noindent{$\mathbf{Algorithm~2.}$T$\ell_1$PCA}\\
\midrule
\textbf{Input}: The data matrix $\mathbf{X} \in\mathbb{R}^{d\times n}$, the parameter of\\
\qquad~~        T$\ell_1$-norm $a$, and the number of projection\\
\qquad~~        vectors $m$.\\
\textbf{Output}: The projection matrix $\mathbf{W}$.\\
\textbf{Initialization}: $\mathbf{W}_0\leftarrow\mathbf{\varnothing}$, $\mathbf{T}_0\leftarrow\mathbf{I}$, $\mathbf{X}_0\leftarrow\mathbf{X}$.\\
$j\leftarrow1$.\\
\textbf{Repeat}:\\
\quad $\mathbf{X}_j\leftarrow\mathbf{T}_{j-1}^{T}\mathbf{X}_0$; \\
\quad Solve problem \eqref{TL1 model4} by Algorithm $1$ and get its solution    \\
\quad $\boldsymbol{w}_j$, compute the $j$-th projection vector $\boldsymbol{w}_j\leftarrow\mathbf{T}_{j-1}\boldsymbol{w}_j$;\\
\quad Update $\mathbf{W}_j\leftarrow\lbrack\mathbf{W}_{j-1} ,\boldsymbol{w}_j\rbrack$;  \\
\quad Compute $\mathbf{T}_j$ by solving the linear equations $\mathbf{W}_{j}^T \hbox{T}=0$  \\
\quad and following the Gram-Schmidt procedure; \\
\textbf{Until} $j=m$  \\
\bottomrule
\end{tabular} \label{Algorithm}\\
\end{center}

% the proof of convergence
\begin{prop}\label{Orth Alg2} The projection vectors $\boldsymbol{w}_1,\cdots,\boldsymbol{w}_{m}$ obtained by Algorithm 2 are orthonormal.
\end{prop}
\begin{proof}
According to the inductive assumption, we first assume that vectors $\boldsymbol{w}_1,\cdots,\boldsymbol{w}_{m-1}$ are orthonormal in a $d$-dimensional subspace. Thus $\mathbf{W}_{m-1}=[\boldsymbol{w}_1,\cdots,\boldsymbol{w}_{m-1}]\in\mathbb{R}^{d\times (m-1)}$ is an orthonormal matrix and we denote Span$\mathbf{V}_{m-1}=(\boldsymbol{w}_1,...,\boldsymbol{w}_{m-1})$. Then $\mathbf{V}_{m-1}$ is a $\left( m-1 \right)$-dimensional subspace. Recall that the primary goal to search for a vector $\boldsymbol{w} \in \mathbb{R}^{d}$ satisfying problem \eqref{TL1PCA model3}. Once the subspace $\mathbf{V}_{m-1}$ has been obtained, we need to solve $\boldsymbol{w}_m$ through the following optimization problem, which could be solved by Algorithm 1
\begin{equation}\label{TL1 model4}
\begin{split}
\underset{\boldsymbol{w} \in \mathbf{V}_{m-1}^{\perp}}{\max}{f(\boldsymbol{w})}\\
~~\hbox{s.t.\ }\boldsymbol{w}^{T}\boldsymbol{w}=1,
\end{split}
\end{equation}
where $\mathbf{V}_{m-1}^{\perp}$ is the null space of $\mathbf{V}_{m-1}$ and dim$\mathbf{V}_{m-1}^{\perp}=d-m+1$. Then update $\boldsymbol{w}_m\leftarrow\mathbf{T}_{m-1}\boldsymbol{w}_m$, where $\mathbf{T}_{m-1} \in \mathbb{R}^{d\times (d-m+1)}$. It is obvious that $\boldsymbol{w}_m$ is orthogonal to $\boldsymbol{w}_j,j=1,...,m-1.$ Therefore, $\mathbf{W}_{m}=[\boldsymbol{w}_1,\cdots,\boldsymbol{w}_{m}]\in\mathbb{R}^{d\times m}$ is also an orthonormal matrix.

In nature, The data is projected onto the subspace $\mathbf{V}_{m-1}^{\perp}$ to implement Algorithm 1. At last, to perform the next iteration, we need to find a basis $\mathbf{T}_{m}=(\boldsymbol{t}_1,...,\boldsymbol{t}_{d-m}) \in \mathbb{R}^{d\times (d-m)}$ of $\mathbf{V}_{m}^{\perp}$ . To obtain $\mathbf{T}_{m}$, we need only to solve the linear equation $\mathbf{W}_{m}^{T}\hbox{T}=0$ and make this basis orthonormal by following the Schmidt orthogonalization.
\end{proof}

\section{Experiments}
In this section, we evaluate the performance of T$\ell_1$PCA on an artifical dataset and two human face databases including Yale\cite{Yale1997} and Jaffe\cite{Jaffe1998}. To demonstrate the robustness of our method, we add outliers in the artifical dataset and random block noise in the face databases. For comparsion, the calssical PCA\cite{Jollife2002}, PCA-$\ell_1$\cite{Nie2011}, PCA-$\ell_p$\cite{Kwak2014}, and $\ell_p$SPCA\cite{Li2017} have also been utilized. We use the nearest neighbor classifier (1-NN) for classification, which assigns a test sample to the class of its nearest neighbor in the training samples. The implementation environment is MATLAB R2017a.

\subsection{A Toy Example}
Firstly, we evaluate the robustness of T$\ell_1$PCA on a two-dimensional artifical dataset, containing 30 data points and 4 outliers. The 30 data points are generated by picking  $x_i$ from -3 to 3 with the same interval and yielding $y_i$ from the Gaussian distribution $N(x_i,1)$, satisfying that the summation over $x_i$, $y_i$ equals to zero, and depicted by navy blue "$\bullet$". 4 outliers are of coordinates [-4,4.8], [-3.7,5.1], [-3.3,6] and [-2.4,5.5], depicted by red "$\ast$". The dataset is shown in Fig. \ref{Toyfig}.

Obviously, when discarding outliers, the included angle between the ideal projection direction of the dataset and $x$-axis is $45{\degree}$, where the ideal projection direction is depicted by black solid line. The first principal components of T$\ell_1$PCA, PCA, PCA-$\ell_1$, PCA-$\ell_p$ and $\ell_p$SPCA are obtained by applying them to the artifical dataset with outliers under different parameters. The parameters $a$ in T$\ell_1$PCA and $p$ in PCA$\ell_p$ and $\ell_p$SPCA are chosen from $a=100, 1, 0.01$ and $p= 1, 0.5, 0.01$, respectively. These principal components and their included angles with the ideal projection direction are also plotted in Fig. \ref{Toyfig}.

From Fig. \ref{Toyfig}, we see that the principal components learned by PCA, PCA-$\ell_1$, PCA-$\ell_p$ and $\ell_p$SPCA are severely deviated from the ideal projection direction, and the included angle of $\ell_p$SPCA is up to $35.1^\circ$ when $p = 0.01$. However, the principal components learned by T$\ell_1$PCA are slightly deviated from the ideal projection direction, especially when paramater $a$ is small. Its principal components are much closer to the ideal projection direction which indicate that T$\ell_1$PCA is more robust to outliers than the other PCAs.

\begin{figure}[!htbp]
\centering
\subfigure[PCA and PCA-$\ell_1$]{\includegraphics[width=0.17\textheight]{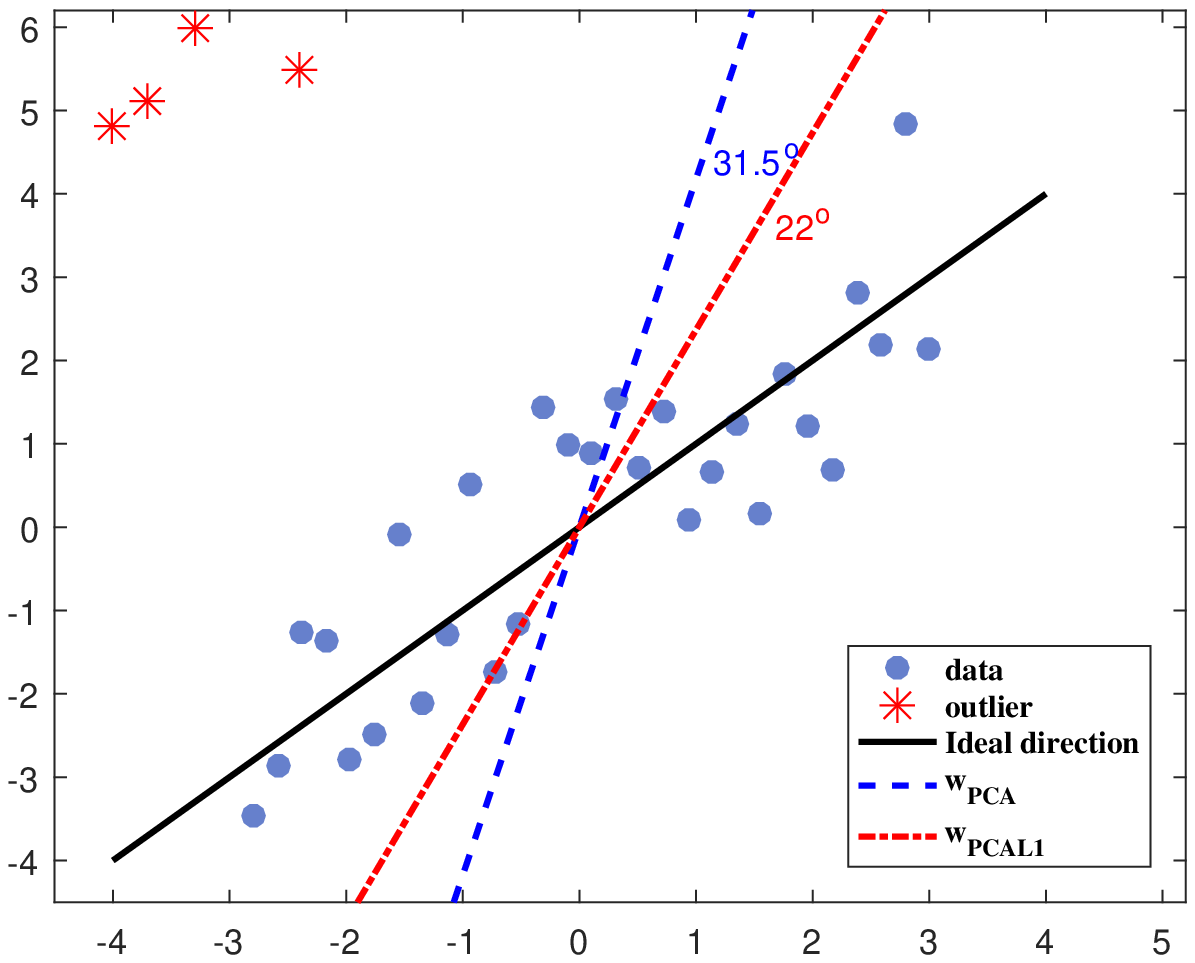}}
\subfigure[PCA-$\ell_p$]{\includegraphics[width=0.17\textheight]{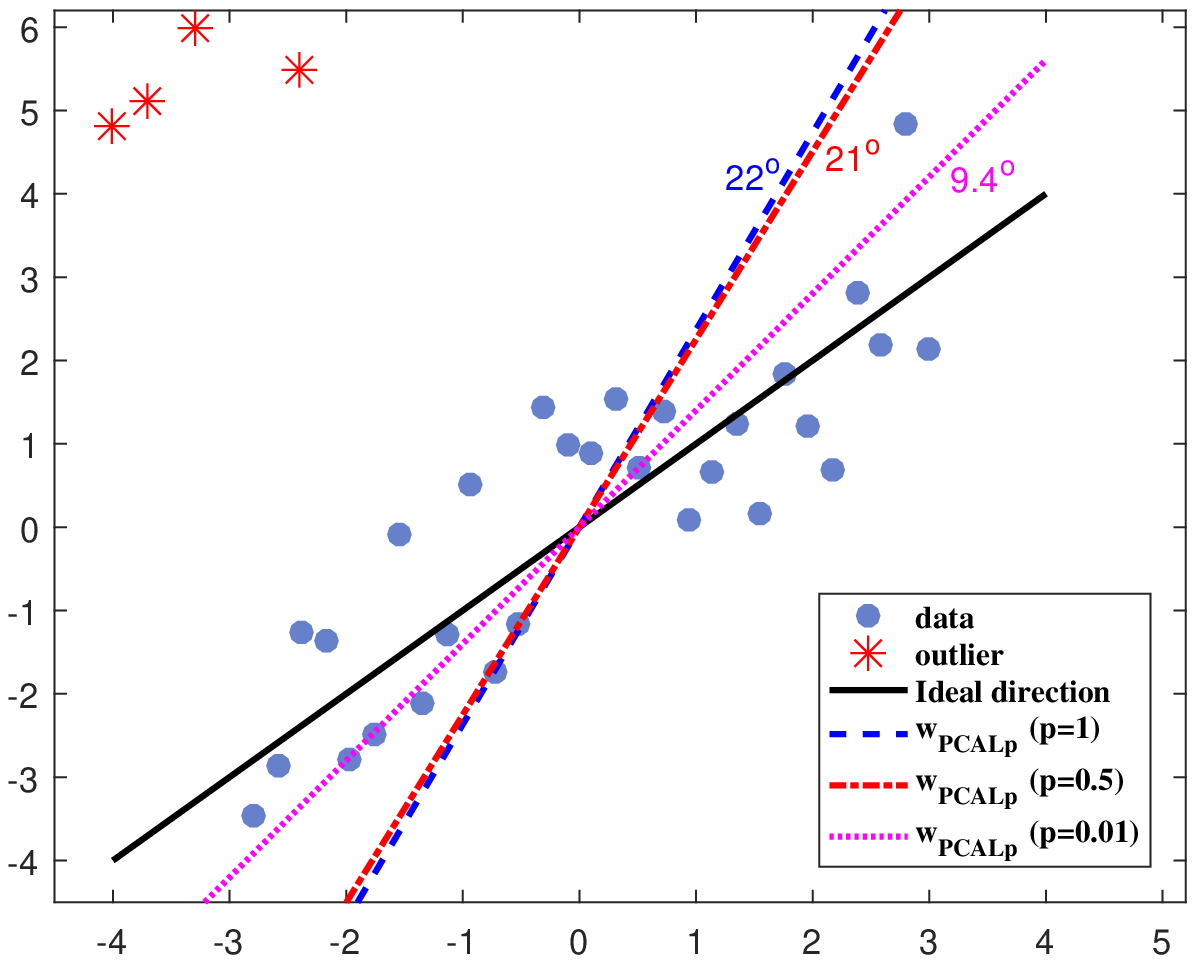}}\\
\subfigure[$\ell_p$SPCA]{\includegraphics[width=0.17\textheight]{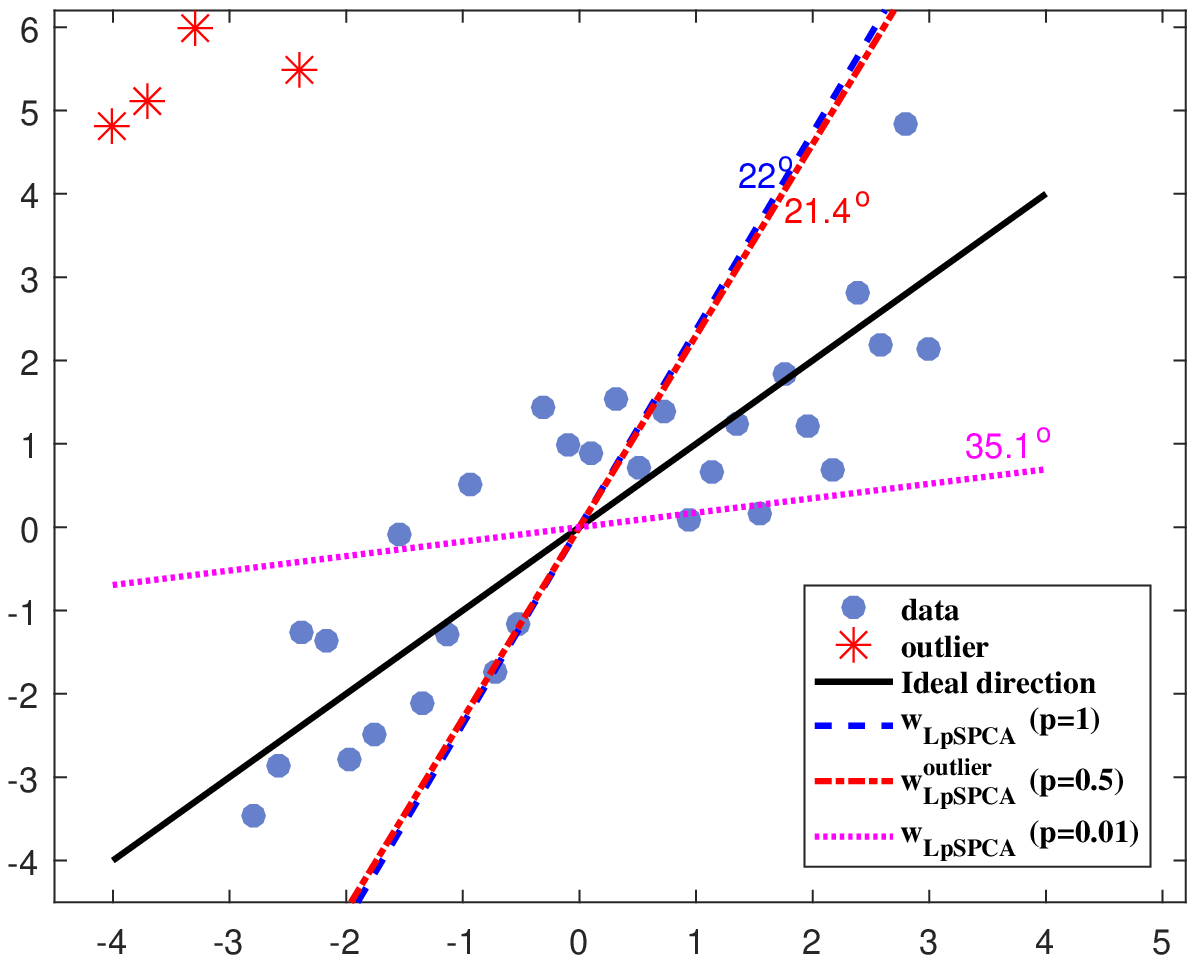}}
\subfigure[T$\ell_1$PCA]{\includegraphics[width=0.17\textheight]{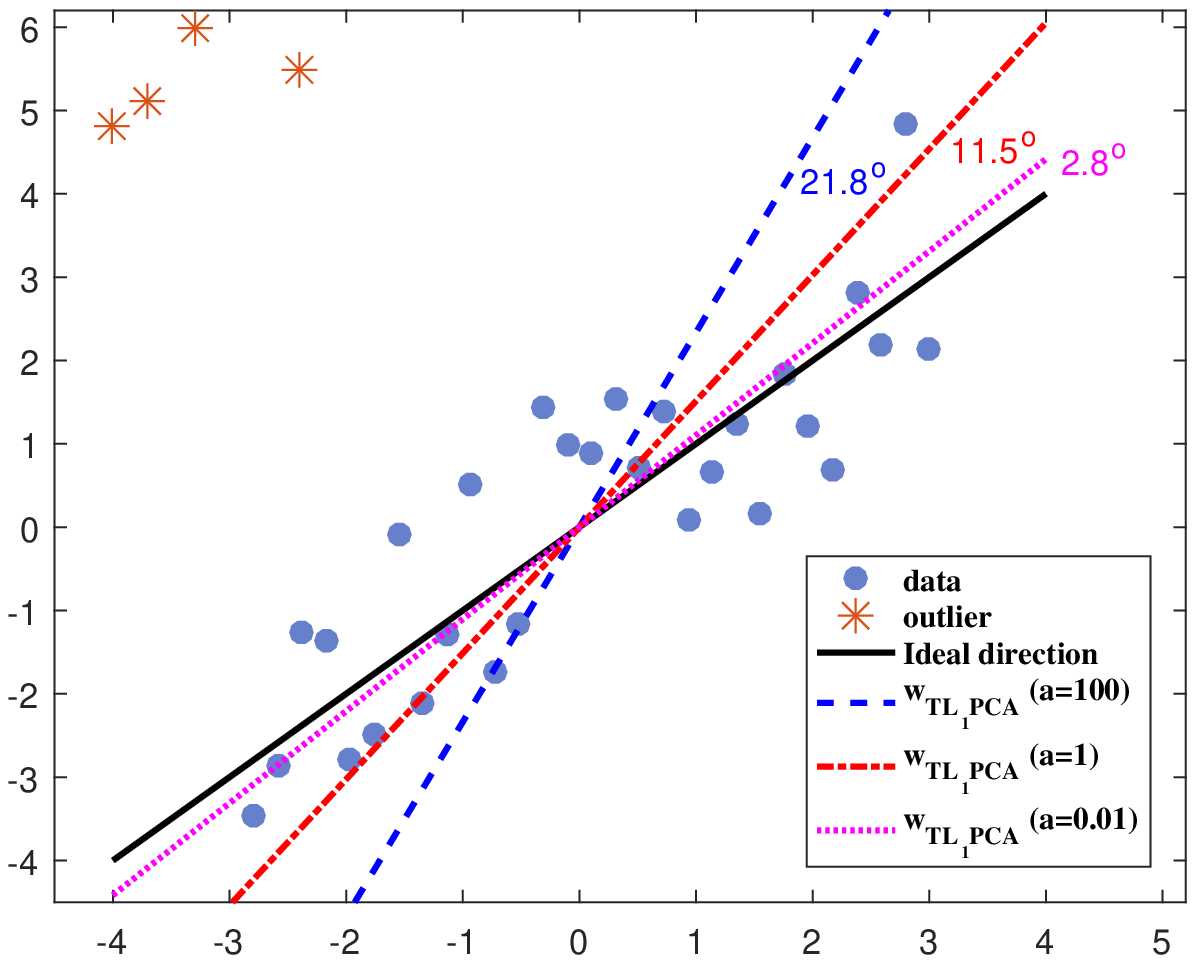}}
\caption{Experimental results for PCA, PCA-$\ell_1$, PCA-$\ell_p$, $\ell_p$SPCA, T$\ell_1$PCA on an artifical dataset.} \label{Toyfig}
\end{figure}

\subsection{Real-world Datasets}
The performance of T$\ell_1$PCA, PCA$\ell_p$ and $\ell_p$SPCA depends on parameter $a$ or $p$. For each of the three methods, we search the optimal parameter from $a$ = [100, 50, 10, 1, 0.5, 0.1, 0.05, 0.01, 0.001]  or $p$ = [1, 0.9, 0.7, 0.5, 0.3, 0.1, 0.01, 0.001] on all real-world datasets.

\subsubsection{Yale}
The Yale face database contains 165 grayscale images of 15 individuals under different lighting conditions and facial expressions, these facial expressions include happy, normal, sad, sleepy, surprised and wink. Each individual has 11 images. Each image in Yale database is cropped to $32\times32$ pixels. 9 images of each person are randomly selected for training and the $i \times i$ ($i=8$ and $12$) block noise is added to them. Original and noisy sample images of one individual are shown in Fig. Then we employ PCA, PCA-$\ell_1$, PCA-$\ell_p$, $\ell_p$SPCA and T$\ell_1$PCA to extract features respectively. For each given parameter $a$ or $p$, we compute the average classification accuracies of 15 random splits on original data and $i\times i$ noisy data.

%\begin{figure}[!htbp]
%\centering
%\includegraphics[width=0.4\textwidth,height=3.3cm]{Yaleface.jpg}\hfill
%\caption{The first row is the original sample images from Yale database. The second row is $8\times8$ noisy images. The third row is $12\times12$ noisy images.}
%\label{Yaleface}
%\end{figure}

\begin{figure*}[!t]
\centering
\subfigure[]{\includegraphics[width=0.24\textheight,height=4.52cm]{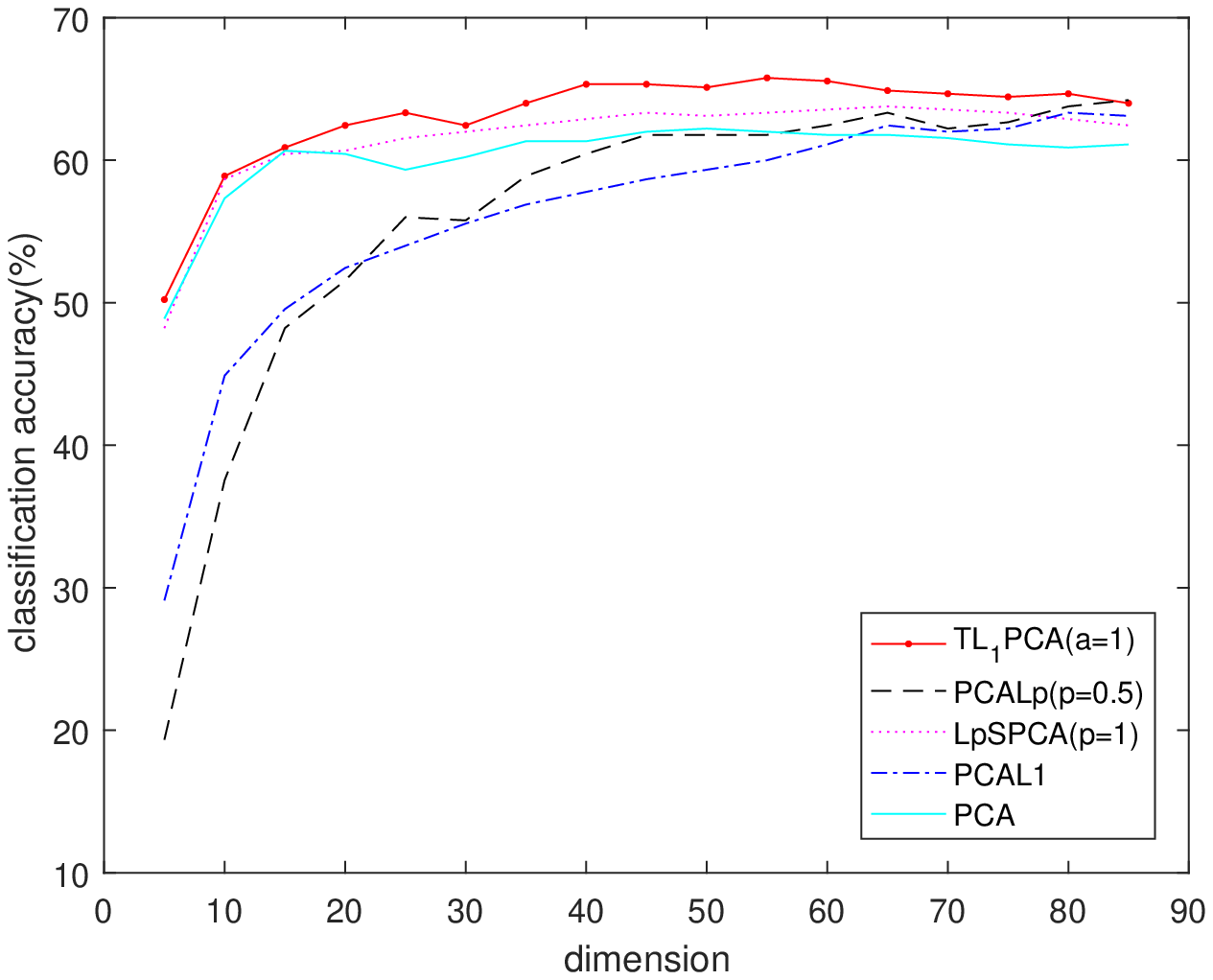}}
\subfigure[]{\includegraphics[width=0.24\textheight,height=4.52cm]{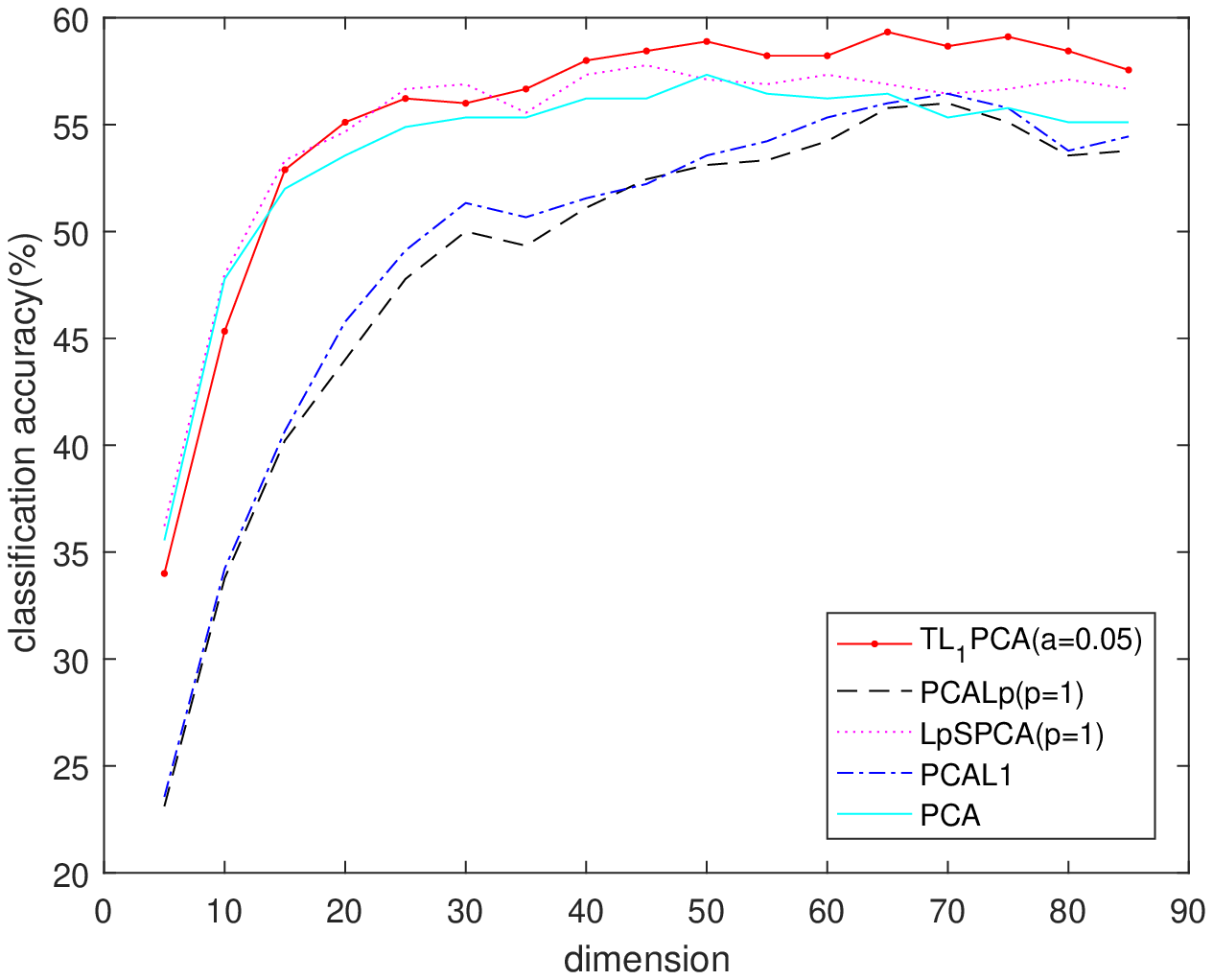}}
\subfigure[]{\includegraphics[width=0.24\textheight,height=4.52cm]{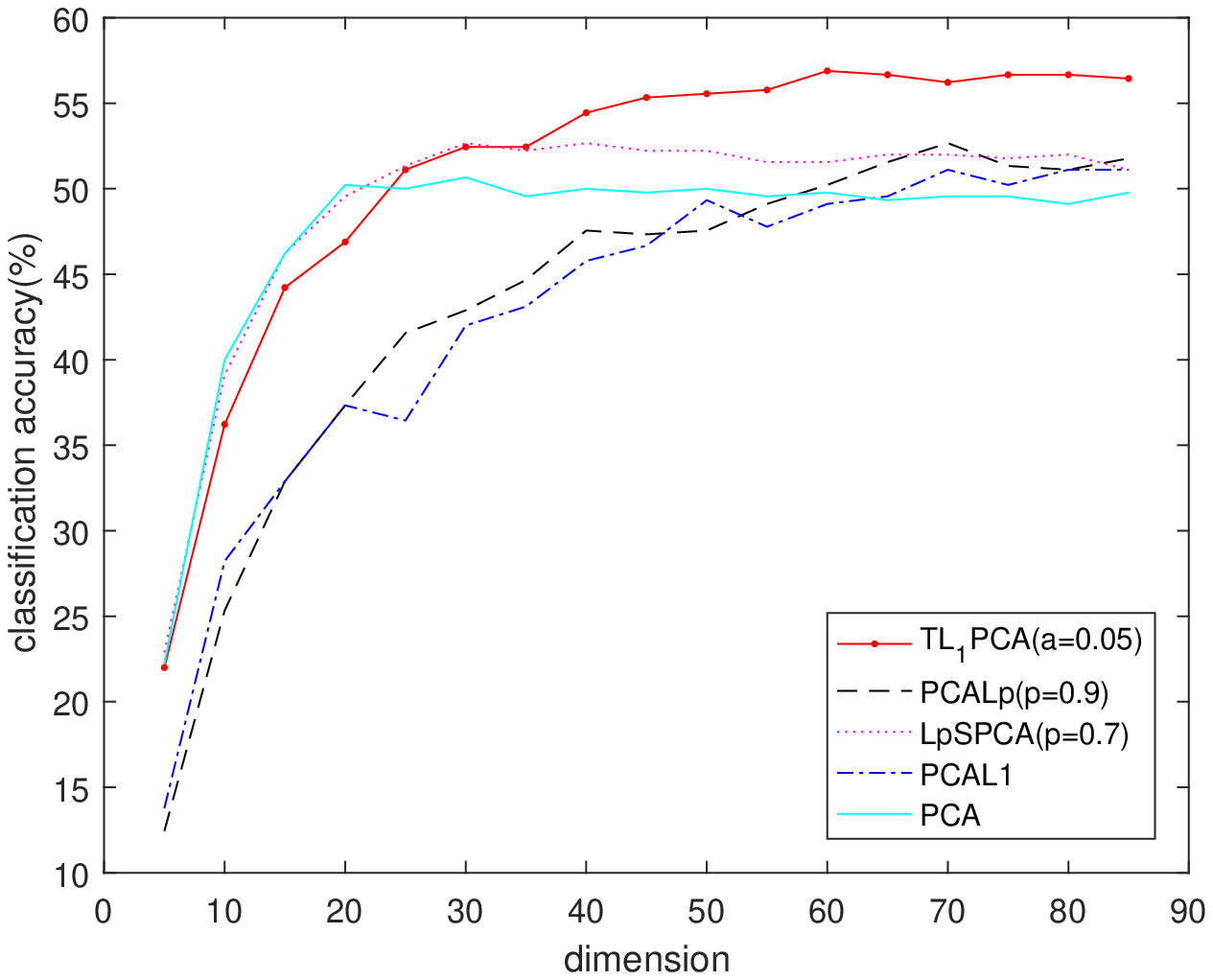}}
\caption{The accuracies of Yale database under the optimal parameter. (a) The accuracy of each method on original data. (b) The accuracy of each method on data with $8\times8$ block noise. (c) The accuracy of each method on data with $12\times12$ block noise.} \label{Yalefig}
\end{figure*}

\begin{table}[!htbp]
\begin{center}
\caption{The average classification accuracies of yale database under the optimal dimension.}
\setlength{\tabcolsep}{1mm}{
\begin{tabular}{cccccc}
\toprule
& \multicolumn{5}{c}{Accuracy(\%)}\\
\cline{2-6}
%Method & \ original data & \ with $8\times8$ block noise & \ with $12\times12$ block noise    \\
Method & \ T$\ell_1$PCA & \ PCA$\ell_p$ & \ $\ell_p$SPCA & \ PCA$\ell_1$ & \ PCA   \\
\hline
Original data                 &\textbf{65.77} &64.22 &63.77 &63.33 &62.22\\
With $8\times8$ block noise   &\textbf{59.33} &56.00 &57.78 &56.44 &57.33\\
With $12\times12$ block noise &\textbf{55.33} &52.66 &52.00 &51.11 &50.66\\
\bottomrule
\end{tabular}
}
\label{Yaleacc}
\end{center}
\end{table}

For each method, Fig. \ref{Yalefig} plots their average classification accuracy vs. the dimension of reduced space under the optimal parameter. Table \ref{Yaleacc} lists the classification accuracy of each method under the optimal dimension. The above results show that T$\ell_1$PCA outperforms the other methods in all conditions.
And the accuracy of T$\ell_1$PCA is around 2.5\% higher than the other PCAs. From Fig. \ref{Yalefig}, the accuracy of T$\ell_1$PCA has an upward tendency along the number of dimension, comparing with data without noise, the advantages in performance are strengthened on noisy data. The reason is that we use T$\ell_1$-norm, which has stronger suppression effect to noise. When the number of dimension reaches around 30, the accuracy tends to be stable.

\subsubsection{Jaffe}
The Jaffe database contains 213 images of 7 facial expressions posed by 10 Japanese female individuals. Each image is resized to $32\times32$ pixels. We randomly choose 70\% of each individual's images for training, adding the same noise as Yale database, and the remainders for testing. Some samples in Jaffe database are shown in Fig . Then PCA, PCA-$\ell_1$, PCA-$\ell_p$, $\ell_p$SPCA and T$\ell_1$PCA are applied to extract features. For each given parameter $a$ or $p$, the average classification accuracies on original data and $i\times i$ ($i=8$ and $12$) noisy data over 15 random splits are considered.
%\begin{figure}[!htbp]
%\centering
%\includegraphics[width=0.4\textwidth,height=3.3cm]{Jaffeface.jpg}\hfill
%\caption{The first row is the original sample images from Jaffe database. The second row is $8\times8$ noisy images. The third row is $12\times12$ noisy images.}
%\label{Jaffeface}
%\end{figure}

\begin{figure*}[!bt]
\centering
\subfigure[]{\includegraphics[width=0.24\textheight,height=4.52cm]{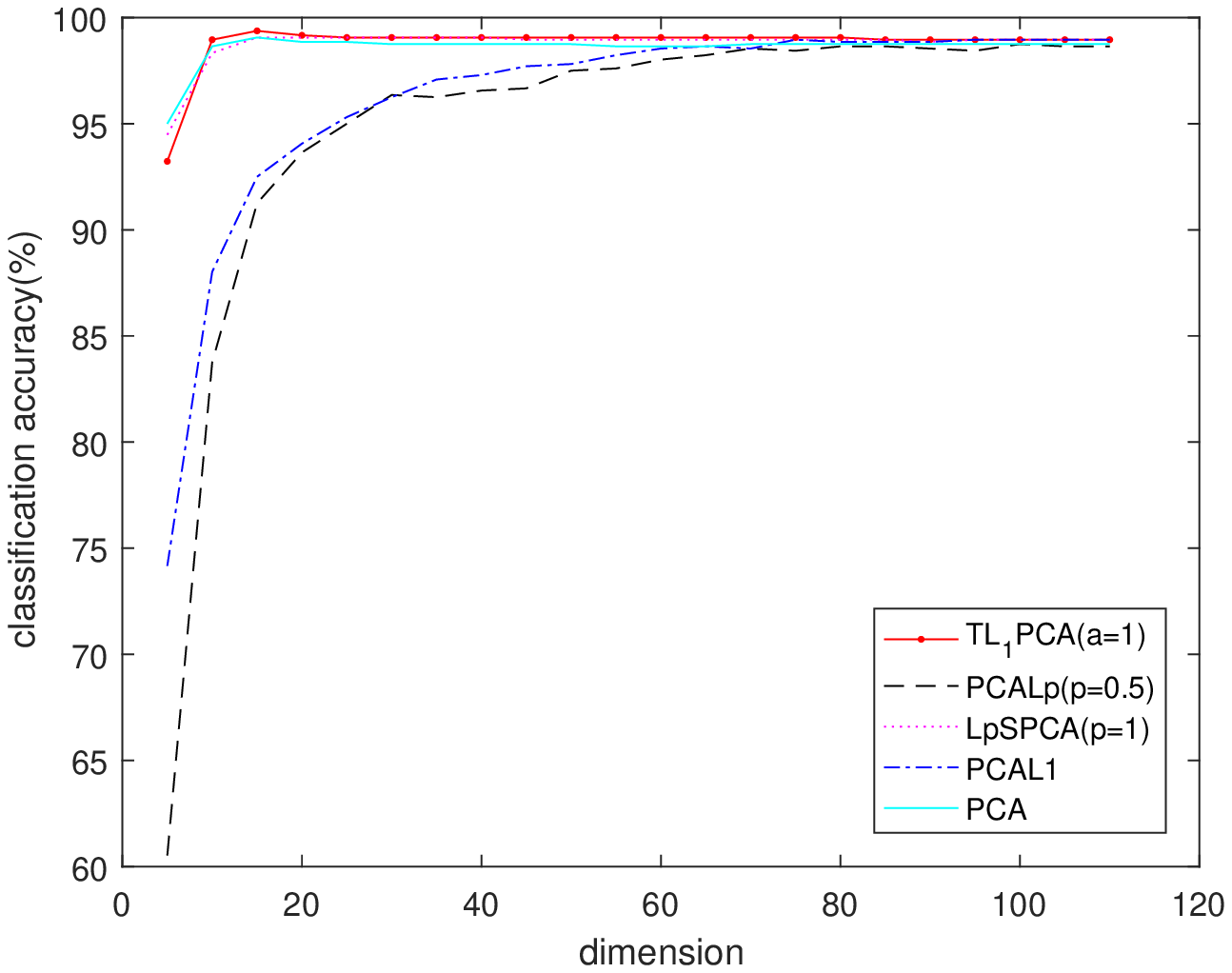}}
\subfigure[]{\includegraphics[width=0.24\textheight,height=4.52cm]{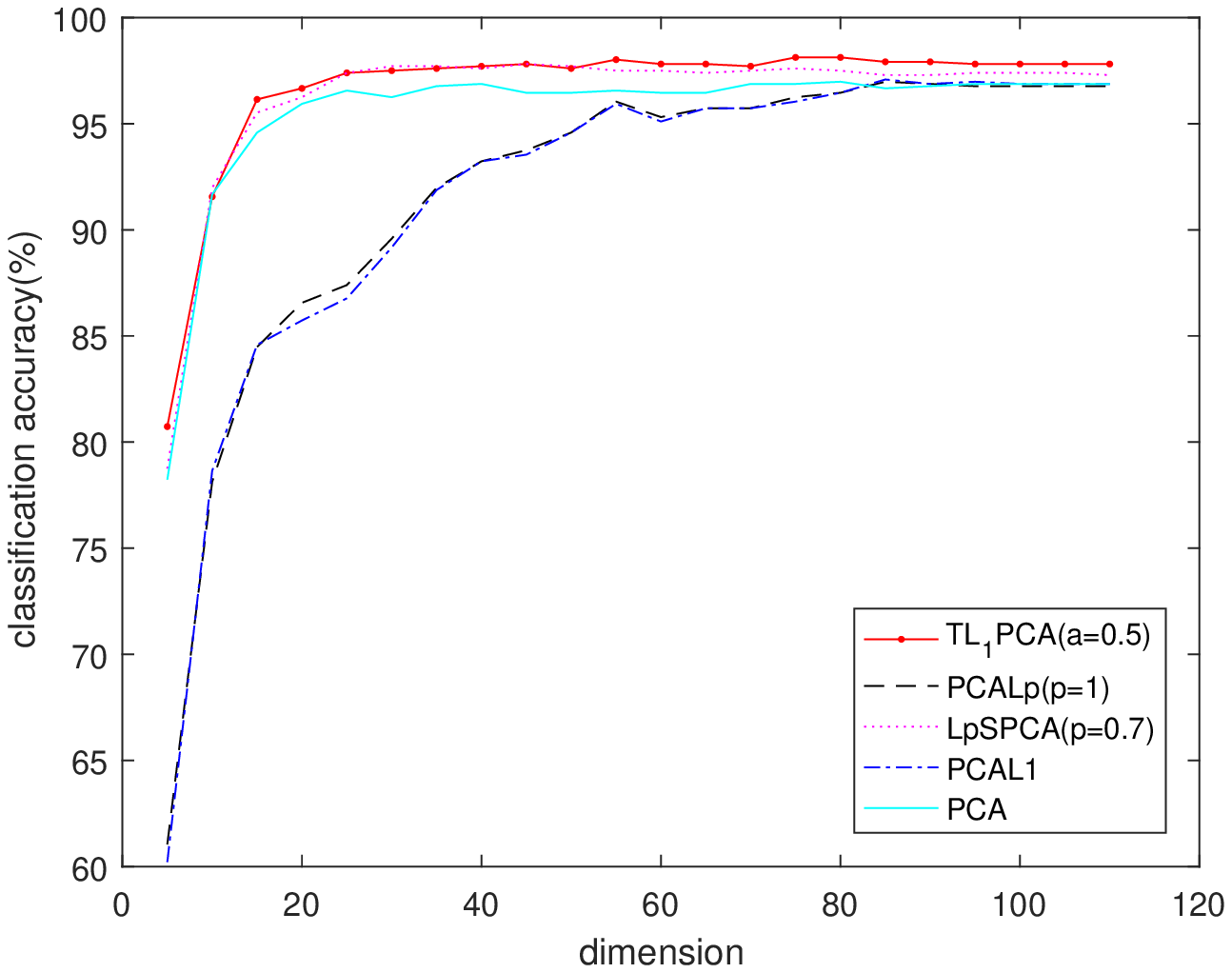}}
\subfigure[]{\includegraphics[width=0.24\textheight,height=4.52cm]{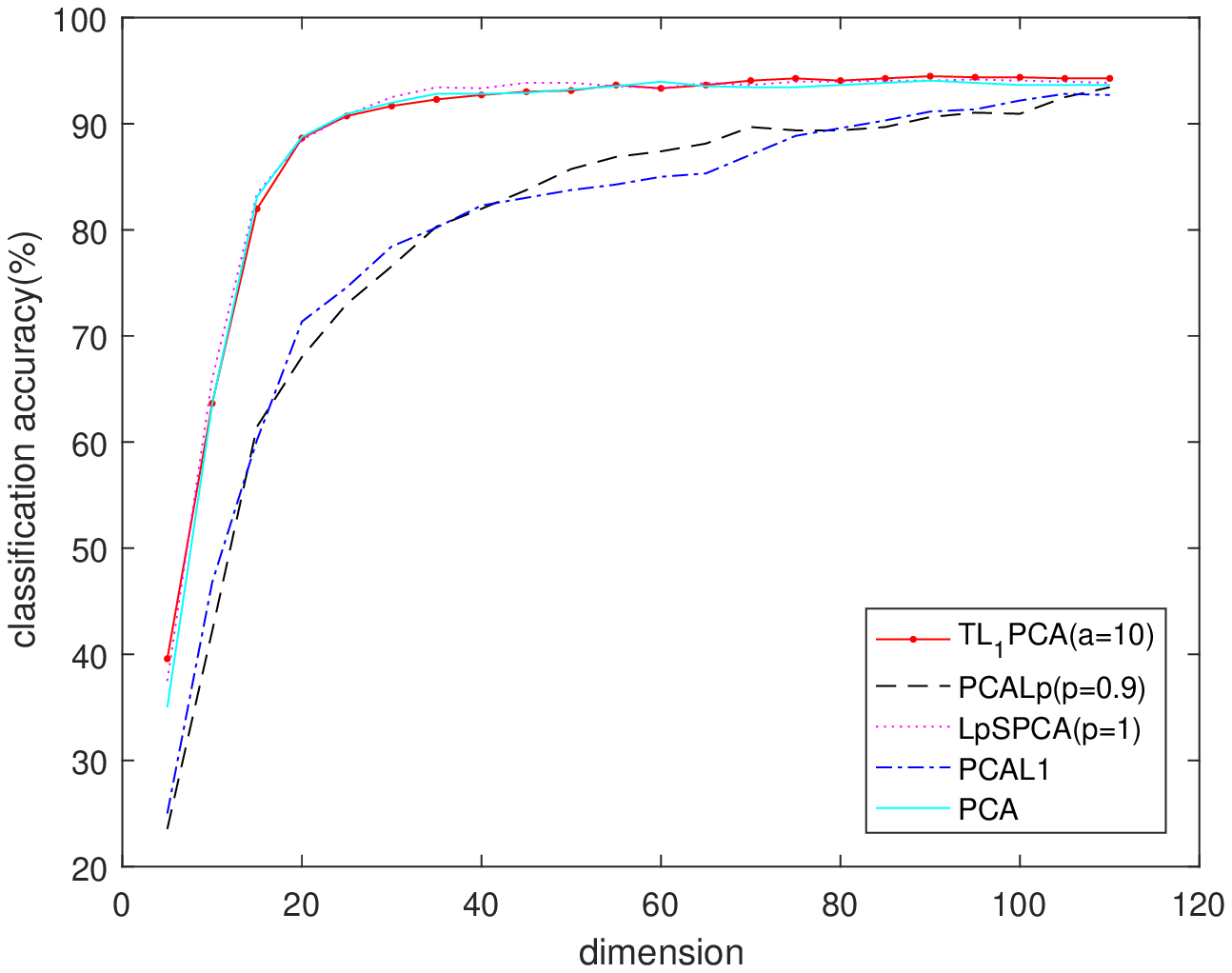}}\\
\caption{The accuracies of Jaffe database under the optimal parameter. (a) The accuracy of each method on original data. (b) The accuracy of each method on data with $8\times8$ block noise. (c) The accuracy of each method on data with $12\times12$ block noise.} \label{Jaffefig}
\end{figure*}

\begin{table}[!htbp]
\begin{center}
\caption{The average classification accuracies of jaffe database under the optimal dimension.}
\setlength{\tabcolsep}{1mm}{
\begin{tabular}{cccccc}
\toprule
& \multicolumn{5}{c}{Accuracy(\%)}\\
\cline{2-6}
%Method & \ original data & \ with $8\times8$ block noise & \ with $12\times12$ block noise    \\
Method & \ T$\ell_1$PCA & \ PCA$\ell_p$ & \ $\ell_p$SPCA & \ PCA$\ell_1$ & \ PCA   \\
\hline
Original data                 &\textbf{99.37} &98.75 &99.06 &98.95 &99.06\\
With $8\times8$ block noise   &\textbf{98.12} &96.97 &97.81 &97.08 &96.97\\
With $12\times12$ block noise &\textbf{94.47} &93.43 &94.16 &92.81 &94.06\\
\bottomrule
\end{tabular}
}
\label{Jaffeacc}
\end{center}
\end{table}

Fig plots the average classification accuracy vs. the dimension of reduced space for each method under the optimal parameter on Jaffe database. The  classification accuracy of each method under the optimal parameter is listed in Table \ref{Jaffeacc}. It can be seen that T$\ell_1$PCA is superior to the other methods. And the trend of accuracy and the behaviour of parameter $a$ on this database are also similar to those on Yale database. When the number of dimension reaches about 20, the accuracy tends to be stable.

\subsection{Convergence Experiments}
We finally investigate the performance of T$\ell_1$PCA in terms of convergency. Fig. \ref{ToyCurve_fig} plots the convergence curves on artifical data with/without outliers and the above two databases with/without noise. The results illustrate that T$\ell_1$PCA can converge quickly, generally within about 10 steps. It is consistent with the conclusion in proposition 1.
\begin{figure}[!htbp]
\centering
\subfigure[Artifical dataset]{\includegraphics[width=0.116\textheight]{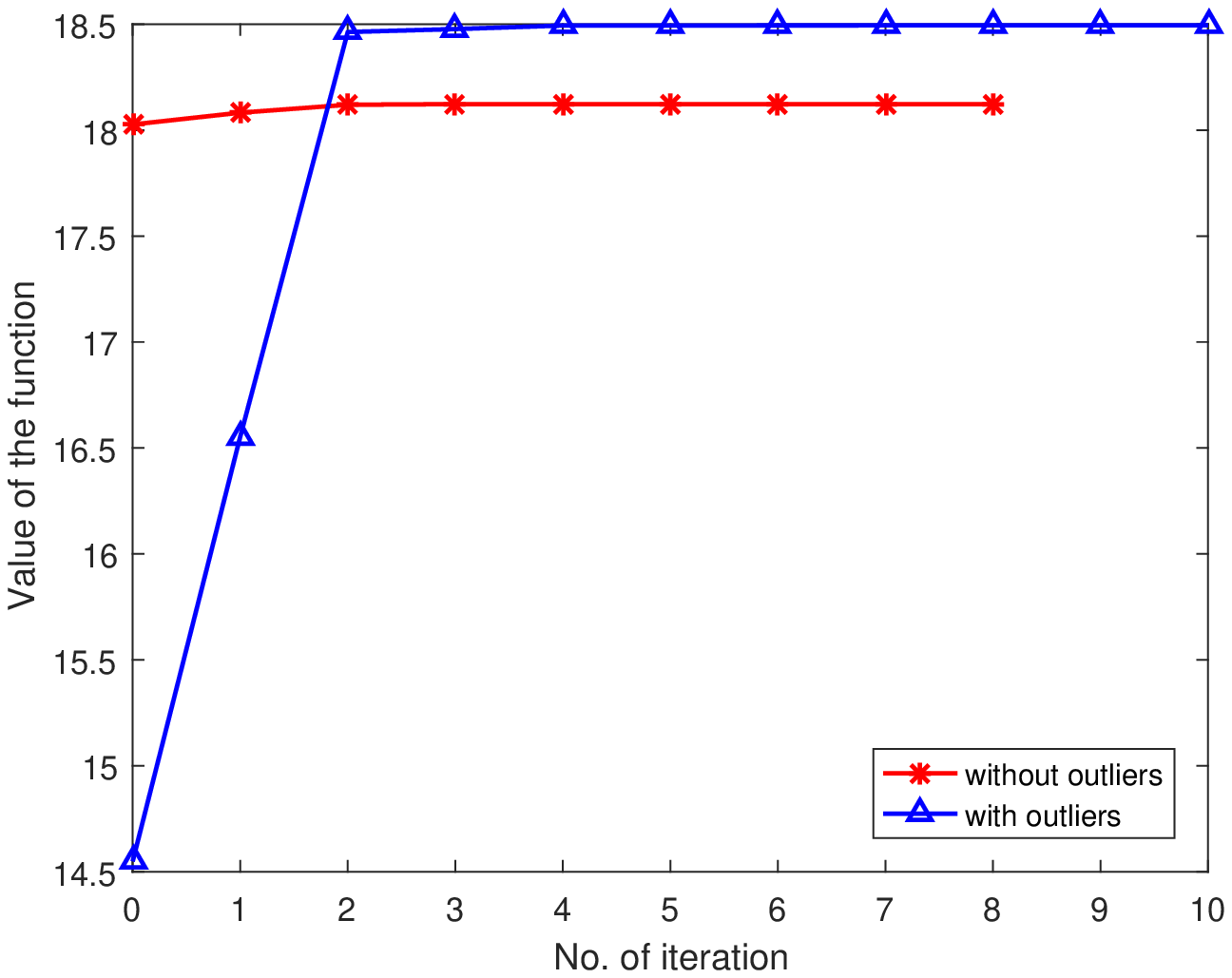}}
\subfigure[Yale]{\includegraphics[width=0.116\textheight]{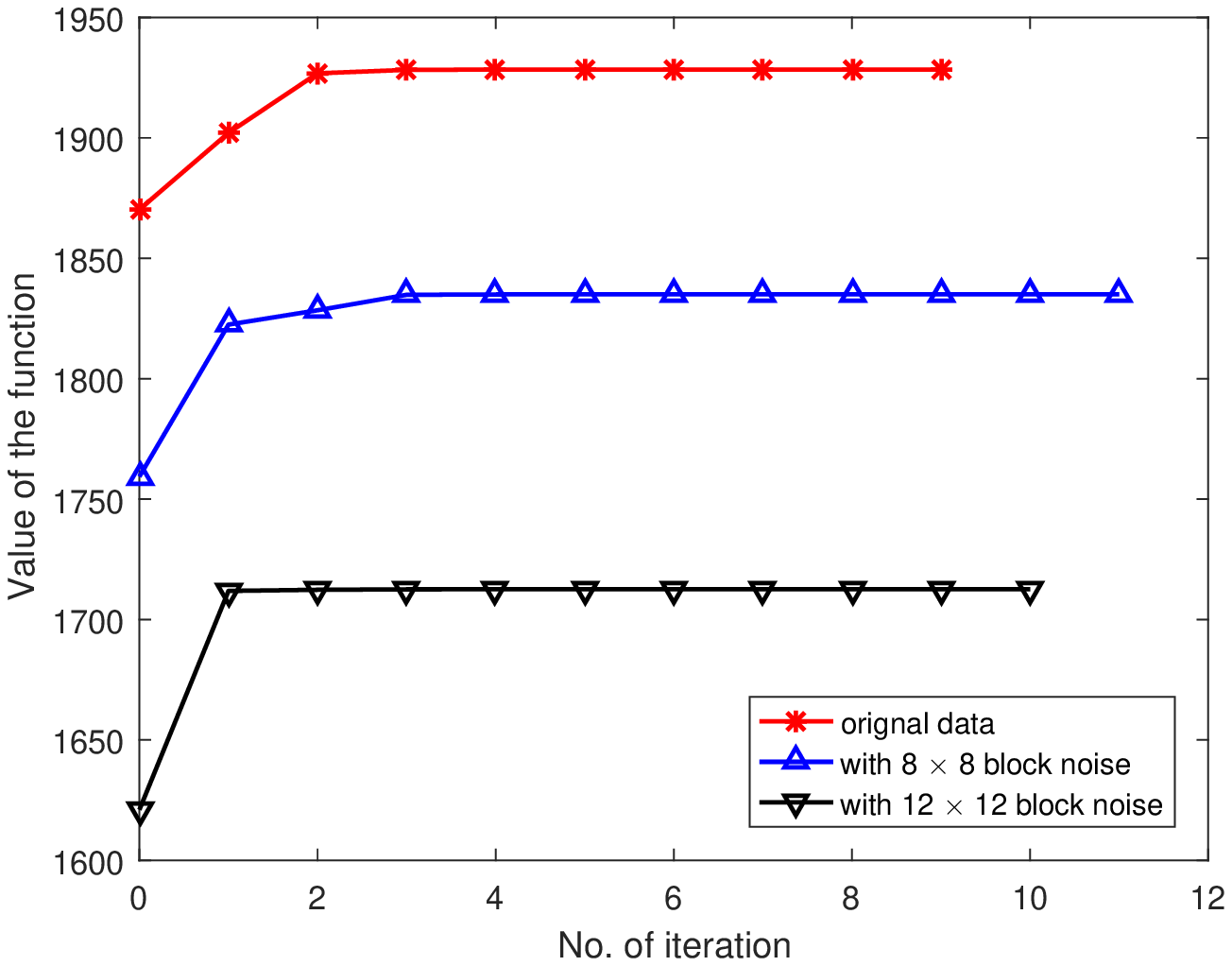}}
\subfigure[Jaffe]{\includegraphics[width=0.116\textheight]{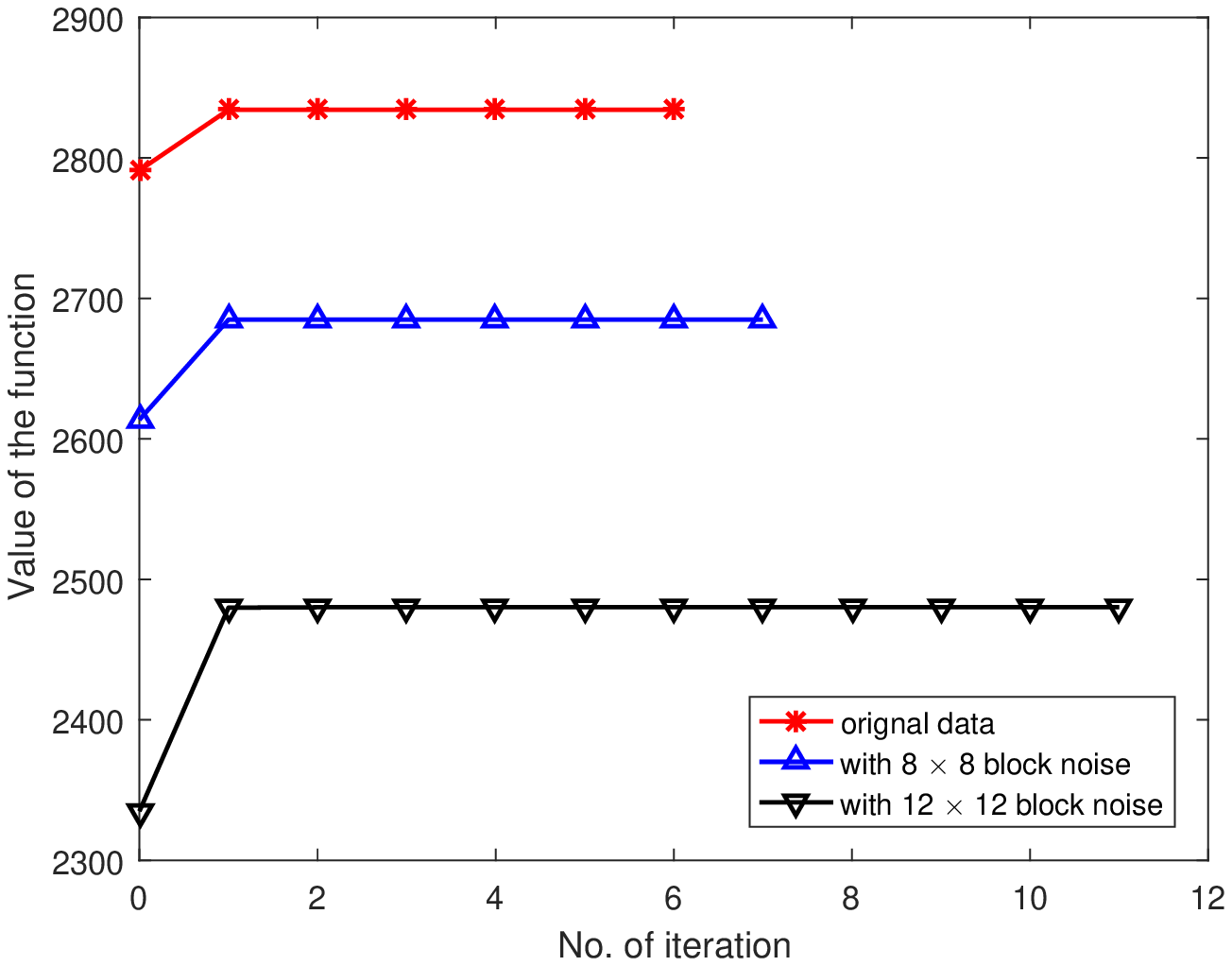}}
\caption{Variation of objective function value along the number of iteration for T$\ell_1$PCA on artifical dataset, Yale database and Jaffe database.} \label{ToyCurve_fig}
\end{figure}

\section{Conclusion}
In this paper, we have introduced a new T$\ell_1$-norm and shown its properties which indicate that T$\ell_1$-norm is more robust than $\ell_p$-norm ($0<p<1$). Then we proposed a novel dimensionality reduction method called T$\ell_1$PCA. It employed T$\ell_1$-norm as the distance metric to maximize the dispersion of the projected data. T$\ell_1$PCA was more robust to noise and outliers than $\ell_p$-norm-based PCA methods with higher classification accuracy. And convergence experiments showed that T$\ell_1$PCA can converge quickly. T$\ell_1$-norm not only could be applied to the unsupervised PCA but also supervised dimensionality methods, even other methods in machine learning. These will be our future work.

% or


%\appendices

%\section{}


\ifCLASSOPTIONcaptionsoff
  \newpage
\fi



%\bibliographystyle{IEEEtran}


%\vfill


\end{document}